\documentclass[]{bytedance_seed}

\usepackage{pifont}
\newcommand{\cmark}{\ding{51}}
\newcommand{\xmark}{\ding{55}}
\usepackage[utf8]{inputenc}
\usepackage{amsfonts}
\usepackage{amsmath}
\usepackage{amssymb}
\usepackage{amsthm}
\usepackage{nicefrac}
\usepackage{makecell}
\usepackage{colortbl}
\definecolor{rowgray}{gray}{0.90}
\usepackage{algorithm}
\usepackage{algpseudocode}
\usepackage{float}
\usepackage{wrapfig}
\usepackage{needspace}
\usepackage{comment}

\microtypesetup{expansion=false}
\titleformat*{\paragraph}{\normalfont\bfseries}
\DeclareRobustCommand{\seedtextit}[1]{{\normalfont\itshape #1}}

\hypersetup{
  pdftitle={TriTD: Tri-Partite Trajectory-Distribution Distillation for Real-Time Autoregressive Video Generation},
  pdfauthor={Jia Li, Xurui Peng, Haowei Zhu, Tianyu Zhao, Jiexi Wang, Youwei Zheng, Yuxi Ren, Xinglong Wu, Xing Wang, Hayden Kwok-Hay So}
}

\newtheorem{lemma}{Lemma}
\newtheorem{corollary}{Corollary}

\title{TriTD: Tri-Partite Trajectory\textendash{}Distribution Distillation for Real-Time Autoregressive Video Generation}

\author[1,2,\ast]{Jia Li}
\author[2,\ast]{Xurui Peng}
\author[3]{Haowei Zhu}
\author[2]{Tianyu Zhao}
\author[2]{Jiexi Wang}
\author[2]{Youwei Zheng}
\author[2]{Yuxi Ren}
\author[2]{Xinglong Wu}
\author[2,\dagger]{Xing Wang}
\author[1,\dagger]{Hayden Kwok-Hay So}

\affiliation[1]{The University of Hong Kong}
\affiliation[2]{ByteDance Seed}
\affiliation[3]{Tsinghua University}

\contribution[\ast]{Equal contribution}
\contribution[\dagger]{Corresponding authors.}

\abstract{
Autoregressive (AR) video diffusion has emerged as a leading paradigm for efficient video generation, yet its practical deployment remains bottlenecked by the trade-off between generation quality and inference efficiency. We attribute this bottleneck to three error sources in AR rollouts: the base model error, the train--inference distribution mismatch, and the upstream error propagation. Building on this decomposition, we derive closed-form upper bounds on the per-step and cumulative errors, providing a principled basis for jointly optimizing generation quality and inference efficiency. Guided by this analysis, we propose \textbf{TriTD}, a \textbf{Tri}-Partite \textbf{T}rajectory--\textbf{D}istribution co-optimized distillation framework with three complementary modules. First, for ODE initialization, we adopt a Next-Node Trajectory Distillation (NTD) objective, which shortens the supervision horizon and admits a lower regression-error floor. Second, we introduce Chunk-Graded Noising (CGN), a monotonically non-decreasing per-chunk noising schedule that helps narrow the expected train--inference noise mismatch and removes redundant KV-cache updates, yielding substantial speedups at no quality cost. Third, we propose Orthogonal-Parallel DMD (OP-DMD), a distribution-matching loss that retains the original DMD supervision while introducing an additional regression constraint in the orthogonal subspace to suppress off-axis drift, improving both training stability and final generation quality. Extensive experiments show that TriTD consistently surpasses existing baselines in visual quality while reaching a real-time throughput of 26.8 FPS on a single H100 GPU, delivering a balanced quality--efficiency solution for interactive streaming video generation.
}

\begin{document}

\maketitle

\section{Introduction}
\label{sec:introduction}

Video generation has made remarkable progress in recent years \cite{openai2024sora,polyak2025moviegencastmedia,klingteam2025klingomnitechnicalreport,deepmind2025veo3,wan2024wan}, enabling a wide range of applications such as text-to-video synthesis \cite{ma2025stepvideot2vtechnicalreportpractice} and image-to-video animation \cite{blattmann2023stablevideodiffusionscaling, hu2024animateanyoneconsistentcontrollable,hu2025animate2highfidelitycharacter}. Research efforts have increasingly focused on improving the visual quality, inference efficiency, and longer-horizon video generation \cite{li2026trainshortinferencelong, yuan2026heliosrealrealtimelong}. To meet these demands, autoregressive (AR) video diffusion \cite{sun2025ardiffusionasynchronousvideogeneration, chen2024diffusionforcingnexttokenprediction, liu2025rollingforcingautoregressivelong, causalforcing2025} has emerged as a promising paradigm for real-time interactive video generation, owing to its causal, chunk-wise decoding. While recent works such as Self-Forcing \cite{huang2024selfforcing} have made notable progress, the trade-off between generation quality and inference efficiency remains unresolved.

First, we analyze AR video generation through the lens of error propagation. We decompose the per-chunk inference error into three sources: the base model error, the train--inference distribution mismatch, and the upstream error propagation, and derive a closed-form upper bound on the cumulative error along the generated sequence. This decomposition reveals two underexplored optimization directions: lowering the base model error from its source, and reducing the context noise distribution mismatch between training and inference.

Motivated by these insights, we propose \textbf{TriTD}, a \textbf{Tri}-Partite \textbf{T}rajectory--\textbf{D}istribution co-optimized distillation framework for AR video generation. It comprises three complementary modules. (1) \seedtextit{Next-Node Trajectory Distillation (NTD).} For ODE initialization, we replace conventional endpoint consistency distillation with next-node supervision under an L1 regression loss, improving the initial trajectory fit. (2) \seedtextit{Chunk-Graded Noising (CGN).} To reduce the train--inference distribution mismatch, we introduce a noising scheme that enforces a monotonically non-decreasing training-time noise level across chunks, applied consistently across both training stages. This schedule brings training context noise closer to inference conditions and removes the redundant per-chunk KV Cache update at inference, reducing per-chunk cost without quality degradation. (3) \seedtextit{Orthogonal-Parallel Distribution Matching Distillation (OP-DMD)}. To stabilize distribution-level training, we augment the standard DMD objective with a regression constraint on the residual component orthogonal to the DMD update direction. This auxiliary supervision suppresses off-axis drift, reduces training variance, and leads to more stable optimization and higher final generation quality.

Our main contributions are summarized as follows:
\begin{itemize}
\item We establish a closed-form upper-bound decomposition for the per-chunk and cumulative errors of AR video generation, offering a new analytical perspective on the quality--efficiency trade-off, and provide theoretical motivation for both NTD and CGN within our distillation framework.

\item We propose \textbf{TriTD}, a tri-partite trajectory--distribution co-optimized distillation framework integrating Next-Node Trajectory Distillation (NTD), Chunk-Graded Noising (CGN), and Orthogonal-Parallel DMD (OP-DMD), whose complementary modules jointly improve generation quality and inference efficiency.

\item Extensive experiments demonstrate that TriTD consistently outperforms state-of-the-art baselines while achieving a real-time throughput of 26.8 FPS on a single H100 GPU, offering a practical solution for real-time interactive video generation.
\end{itemize}

\section{Related work}
\label{sec:related}

\paragraph{Diffusion models and rectified flow.} Diffusion models~\citep{ho2020ddpm,song2021scorebased,Rombach_2022_CVPR} have become the mainstream framework for visual generation. Given a clean sample $x_0$ and Gaussian noise $\epsilon \sim \mathcal{N}(0, I)$, the forward process is $x_t = \alpha_t\, x_0 + \sigma_t\, \epsilon$ with noise schedule $\alpha_t, \sigma_t$, and the model $v_\theta(x_t, t)$ is trained by minimizing $\mathbb{E}_{t,\, x_0,\, \epsilon}\!\left[\, \left\lVert v_\theta(x_t,\, t) - v^{\star} \right\rVert_2^{2} \,\right]$.
Rectified Flow~\citep{liu2023rectified, lipman2023flowmatchinggenerativemodeling, esser2024scalingrectifiedflowtransformers} further simplifies the forward process to a linear interpolation  $x_t = (1-t)\, x_0 + t\, \epsilon$ with the corresponding denoising ODE $\frac{\mathrm{d}x_t}{\mathrm{d}t} = v_\theta(x_t,\, t)$, and has become the backbone of many modern generative models \cite{wu2025qwenimagetechnicalreport, hacohen2024ltxvideo, hacohen2026ltx2efficientjointaudiovisual}.

\paragraph{Efficient diffusion inference.} Beyond reducing sampling steps, a broad family of training-free methods accelerates pretrained diffusion models by reusing temporally redundant features or attention outputs, selectively updating tokens, and pruning redundant conditional computation~\citep{ma2024deepcache,zhao2024pab,zou2024toca,liu2025teacache,Fu_2025_ICCV,zhu2026tap}. These methods reduce the total computation along the denoising trajectory without learning a few-step generator and are therefore complementary to diffusion distillation.

\paragraph{Diffusion model distillation.} Iterative denoising remains a core efficiency bottleneck of diffusion models, motivating two complementary distillation paradigms. (1) \textbf{Trajectory distillation} aligns the student to the teacher denoising ODE. Progressive Distillation~\citep{salimans2022progressive} progressively halves the number of sampling steps via cascaded teacher-student rounds. Consistency Models~\citep{song2023consistency} impose the constraint directly on the ODE trajectory, enforcing that model outputs at any two points on the same trajectory remain consistent. Targeting the linear ODE structure of Rectified Flow, InstaFlow~\citep{liu2024instaflow} and SlimFlow~\citep{zhu2024slimflow} further shorten the effective trajectory length via reflow and re-distillation, unlocking the potential of few-step generation. Trajectory distillation provides fine-grained, per-node supervision on the denoising path and exhibits a \seedtextit{mode-covering} (zero-avoiding) behavior. (2) \textbf{Distribution Matching Distillation}~\citep{yin2024dmd} minimizes the reverse-KL divergence between the student distribution $p_\phi$ and the teacher distribution $p_{\text{tea}}$ for efficient distillation and has been widely used to accelerate both diffusion and flow-matching models. DMD2~\citep{yin2024dmd2} and subsequent works further introduce GAN-style discriminators to sharpen high-frequency details. As a reverse-KL objective, DMD is \seedtextit{mode-seeking} (zero-forcing) and provides weak per-node supervision. Moreover, in autoregressive video generation, the training-time context distribution from DMD self-rollout typically diverges from the inference-time context, inducing a non-trivial train--inference gap.

\paragraph{Distilled autoregressive video generation.} Recently, causal AR diffusion has emerged as a popular research direction in video generation, performing chunk-by-chunk sequential generation with block-causal attention~\citep{sun2025ardiffusionasynchronousvideogeneration,chen2024diffusionforcingnexttokenprediction,liu2025rollingforcingautoregressivelong,skyreels2025,magi12025,yuan2026heliosrealrealtimelong}. Among distillation-based approaches, CausVid~\citep{yin2024causvid} combines few-step distillation with block-causal AR decoding, Self-Forcing~\citep{huang2024selfforcing} simulates the AR self-rollout process during training to alleviate exposure bias, and Causal Forcing~\citep{causalforcing2025} further strengthens generation dynamics with an AR teacher. Long-horizon inference introduces additional challenges from accumulated rollout errors and an increasingly expensive context cache. Recent work explores training-free horizon extension, context compression, and memory-efficient rollout mechanisms~\citep{li2026trainshortinferencelong,fu2026ispa,yuan2026heliosrealrealtimelong}. These inference-time techniques are broadly applicable to and complementary with existing chunk-wise AR video generators, including CausVid, Self-Forcing, and Causal-Forcing. However, existing distilled AR methods do not explicitly model the inter-chunk context-noise distribution and typically require an additional forward pass per chunk to refresh the KV cache.

\section{Method}
\label{sec:method}

\subsection{Preliminary}
\label{sec:preliminary}

We adopt the chunk-wise autoregressive paradigm~\citep{yin2024causvid,huang2024selfforcing},
partitioning a video into $C$ equal-length chunks
$\{X^{(1)}, X^{(2)}, \ldots, X^{(C)}\}$ and factorizing the joint distribution as
\begin{equation}
p\!\left(X^{(1:C)}\right)
= \prod_{c=1}^{C} p\!\left(X^{(c)} \,\big|\, X^{(<c)}\right).
\label{eq:joint-factorization}
\end{equation}
Each conditional $p\!\left(X^{(c)} \mid X^{(<c)}\right)$ is modeled by a few-step
student $\Phi_{\theta}$. For chunk $c$, start from pure noise
$X^{(c)}_{t_1}\sim\mathcal{N}(0,I)$ and denoise at $K$ descending timesteps
$t_1=1000>\cdots>t_K>t_{K+1}=0$. Each update ($k=1,\ldots,K$)
is conditioned on the historical context $X^{(<c)}$:
\begin{equation}
X^{(c)}_{t_{k+1}}
= \Phi_{\theta}\!\left(X^{(c)}_{t_{k}},\, t_{k};\, X^{(<c)}\right).
\end{equation}
\paragraph{Historical context as condition.} To avoid recomputing historical chunks, the student adopts block-causal
attention and injects the historical context with KV Cache
$\mathcal{K}^{(<c)}$. Let $t^{c}_{\mathrm{exit}} \in \{t_{1}, \ldots, t_{K}\}$
denote the actual exit timestep of the denoising chain for the $c$-th chunk,
and let $\hat{X}^{(c)}_{0}$ denote the one-step prediction obtained by mapping
$X^{(c)}_{t^{c}_{\mathrm{exit}}}$ directly to $t = 0$ via $\Phi_{\theta}$.
Denoting the KV-cache update operator as $\Psi$, at baseline inference an additional forward pass is performed after the denoising chain terminates
with $\hat{X}^{(c)}_{0}$ at $t = 0$ to refresh the KV cache:
\begin{equation}
\mathcal{K}^{(\leq c)}
= \Psi\!\left(\hat{X}^{(c)}_{0},\, t = 0;\, \mathcal{K}^{(<c)}\right).
\label{eq:kv-update}
\end{equation}
When $t^{c}_{\mathrm{exit}} \neq t_{K}$, the denoising chain is incomplete and
$\hat{X}^{(c)}_{0}$ deviates from the ideal $X^{(c)}_{0}$ by a residual error
correlated with $t^{c}_{\mathrm{exit}}$, which in turn affects the noise level
and distribution of the KV Cache that conditions subsequent chunks. The
original Self-Forcing~\citep{huang2024selfforcing} samples
$t^{c}_{\mathrm{exit}}$ uniformly at random during training for efficiency, so
the full denoising chain is completed only with probability $1/K$, whereas at
inference the full chain is always executed, i.e.,
$t^{c}_{\mathrm{exit}} = t_{K}$. This systematic mismatch in the distribution
of $t^{c}_{\mathrm{exit}}$ between training and inference constitutes one of
the sources of the train-inference distribution mismatch.

\subsection{Cumulative error bound for autoregressive video diffusion}
\label{sec:error-bound}

In this section, we establish a closed-form upper bound on the cumulative
error of AR video generation under an idealized error model, providing theoretical motivation for our design. For analytical convenience, let
$\sigma_{\mathrm{tr}}(c)$ and $\sigma_{\mathrm{inf}}(c)$ denote the
training-time and inference-time context noise levels of the $c$-th chunk,
respectively, modeled as monotonically related to the exit timesteps
$t^{c}_{\mathrm{exit,tr}}$ and $t^{c}_{\mathrm{exit,inf}}$ defined in
Section~\ref{sec:preliminary}, where a larger exit timestep induces a larger residual prediction error and a larger $\sigma$.

We observe that the single-step inference error $e_{c}$ at the $c$-th chunk
admits a three-way decomposition: (i) the \seedtextit{base model error} $\varepsilon$,
 a uniform bound on the model's residual error under matched, error-free context in this analysis; (ii) the
 \seedtextit{train--inference distribution mismatch} arising from the gap
$|\sigma_{\mathrm{inf}}(c) - \sigma_{\mathrm{tr}}(c)|$ between the
inference-time and training-time context noise levels; and (iii) the
\seedtextit{upstream error propagation} term inherited from the cumulative error
of the previous $c\!-\!1$ chunks via the KV Cache. The following lemma
formalizes this decomposition.

\begin{lemma}[Single-step error upper bound]
\label{lem:single-step}
Assume the matched-context residual is bounded by $\varepsilon$, noise-level perturbations contribute at most $\kappa|\sigma_{\mathrm{inf}}(c)-\sigma_{\mathrm{tr}}(c)|$, and the upstream context-error contribution is bounded by $Le_{c-1}$, with $e_0=0$. Then
\begin{equation}
\label{eq:single-step}
e_c \le \varepsilon + \kappa \cdot |\sigma_{\mathrm{inf}}(c) - \sigma_{\mathrm{tr}}(c)| + L \cdot e_{c-1},
\qquad L \ge 0,\; \kappa \ge 0,
\end{equation}
where $\kappa$ scales the distribution-mismatch contribution and $L$ is the Lipschitz constant of $\Phi_{\theta}$.
\end{lemma}

\begin{proof}[Proof sketch]
Apply the triangle inequality to the matched-context residual, noise-level perturbation, and upstream context perturbation. The stated residual and Lipschitz-type bounds give Eq.~\ref{eq:single-step}.
\end{proof}

Recursively unrolling Lemma~\ref{lem:single-step} along the chunk axis and
summing yields the cumulative bound.

\begin{corollary}[Cumulative error bound]
\label{cor:cumulative}
The total error $E_{C} = \sum_{c=1}^{C} e_{c}$
satisfies
\begin{equation}
E_{C} \;\le\; \sum_{c=1}^{C}
\underbrace{\left( \frac{1 - L^{C-c+1}}{1 - L} \right)}_{w_{c}}
\Big( \varepsilon
\;+\; \kappa \cdot \big| \sigma_{\mathrm{inf}}(c) - \sigma_{\mathrm{tr}}(c) \big|
\Big),
\label{eq:cumulative-bound}
\end{equation}
where the position-dependent weight $w_{c} = (1 - L^{C-c+1})/(1 - L)$ captures
error amplification along the autoregressive chain: errors introduced at
smaller $c$ propagate longer and therefore receive larger weight. The complete derivation, together with the corresponding closed forms of $w_c$ for $L=1$ and $L>1$ and a discussion of their effect on the bound, is given in Appendix~\ref{app:proof-total}.
\end{corollary}

The upper bound in \Cref{eq:cumulative-bound} decomposes into two terms:
\begin{equation}
\underbrace{T \;=\; \sum_{c=1}^{C} w_{c} \cdot \varepsilon}_{\text{Model-error term}},
\qquad
\underbrace{N \;=\; \sum_{c=1}^{C} w_{c} \cdot \kappa \cdot \big| \sigma_{\mathrm{inf}}(c) - \sigma_{\mathrm{tr}}(c) \big|}_{\text{Noise-mismatch term}},
\label{eq:bound-split}
\end{equation}

which jointly characterize the optimization surface of autoregressive video
distillation. 

\subsection{TriTD: a tri-partite trajectory-distribution framework}
\label{sec:framework}

Guided by this analysis, we propose
\textbf{TriTD}, a tri-partite trajectory-distribution co-optimization
framework for autoregressive video distillation. As illustrated in
\Cref{fig:framework}, TriTD (i) reshapes the training objective of ODE-initialization stage with Next-Node Trajectory Distillation (NTD) to reduce the base model error $\varepsilon$; (ii) introduces a chunk-graded noising training strategy that shrinks the train-inference noise distribution mismatch; and (iii) adopts the Orthogonal-Parallel DMD (OP-DMD) loss to mitigate unsupervised off-axis drift and training oscillation. The regression-error analysis motivates NTD's next-node supervision, while the expected mismatch reduction motivates Chunk-Graded Noising. OP-DMD complements these components with distribution-level optimization control.

\begin{figure}[t]
  \centering
  \includegraphics[width=0.99\linewidth]{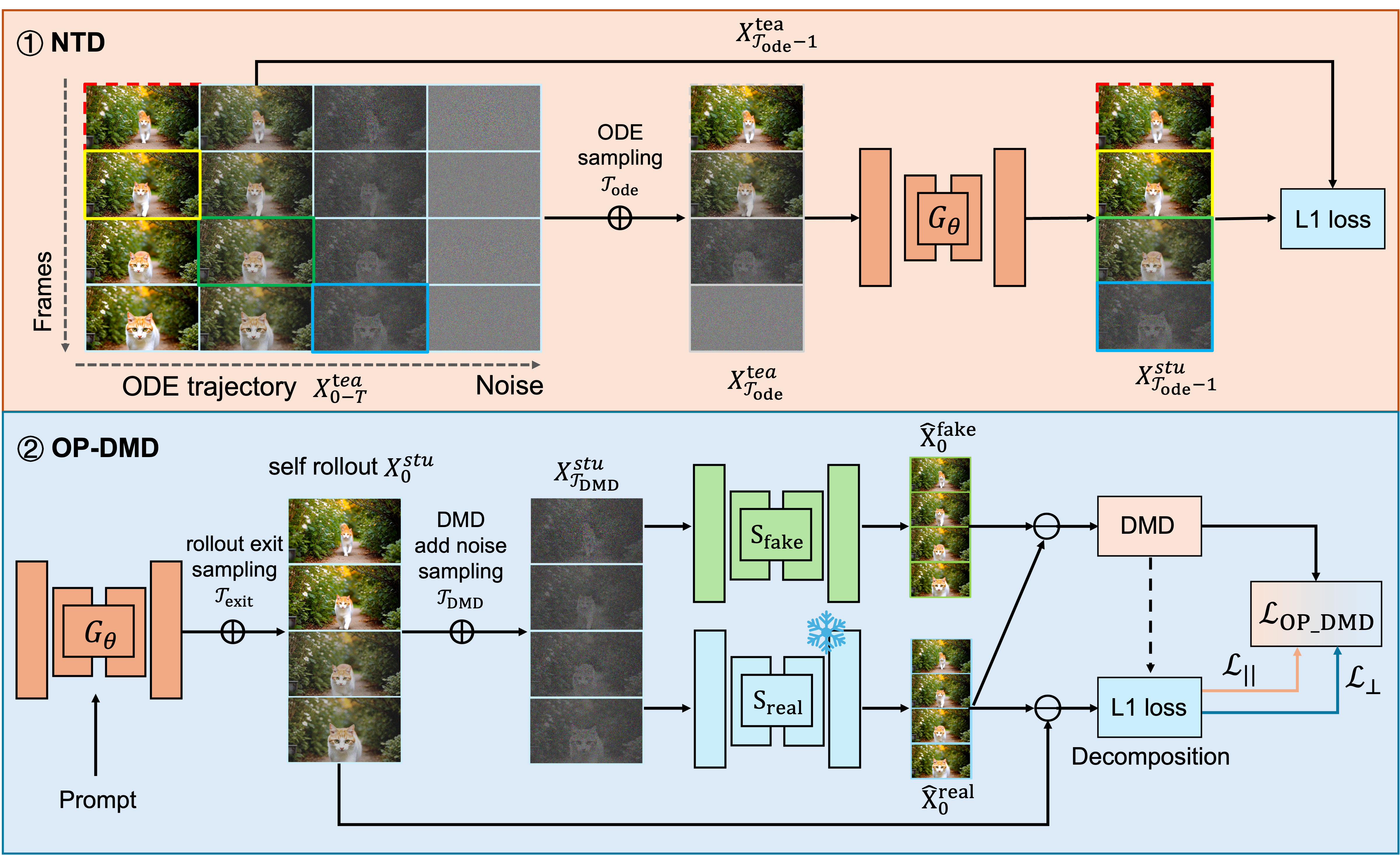}
    \vspace{-1mm}
  \caption{Overview of the proposed TriTD Framework, including (1) NTD-based ODE initialization stage and (2) OP-DMD-based distribution distillation stage. Chunk-Graded noising training is consistently applied throughout both stages. Frames enclosed by solid boxes of the same color form paired student-teacher loss, while frames in red dashed boxes are excluded from loss computation.}
  \vspace{-4mm}
  \label{fig:framework}
\end{figure}

\subsubsection{NTD: tightening the base-model error \texorpdfstring{$\varepsilon$}{epsilon}}
\label{sec:ntd}

Existing ODE-initialization methods typically adopt an
\seedtextit{endpoint-consistency} objective: at any intermediate timestep
$t_{i}$, student $\Phi_{\theta}$ is required to regress the
teacher's clean prediction $x^{\star}_{0}$ at $t=0$:
\begin{equation}
\mathcal{L}_{\mathrm{ode\text{-}init}}
\;=\;
\mathbb{E}_{c,\,i}\!\left[\,
\bigl\|\,\Phi_{\theta}(x_{t_{i}})-x^{\star}_{0}\,\bigr\|^{2}_{2}
\right].
\label{eq:endpoint-l2}
\end{equation}

We propose to replace the endpoint target with the \seedtextit{next node} prediction  $x_{t_{i+1}}^{\star}$ on the teacher trajectory and adopt $L_{1}$ instead of $L_{2}$. The \seedtextit{Next-Node Trajectory Distillation} (NTD) loss is formulated as:
\begin{equation}
\mathcal{L}_{\mathrm{NTD}}
=
\mathbb{E}_{c,\,i}\!\left[\,
\bigl\|\,\Phi_{\theta}(x_{t_{i}})-x^{\star}_{t_{i+1}}\,\bigr\|_{1}
\right].
\label{eq:ntd-l1}
\end{equation}
Whereas \cref{eq:endpoint-l2} regresses the endpoint with MSE,
NTD uses the next saved teacher ODE node and an outlier-robust $L_1$
loss. To motivate this choice, we analyze the regression-error floors
associated with the two objectives.

\begin{lemma}[Comparison of regression-error floors]
\label{lem:ntd-bound}
Assume a locally affine teacher target relation and coordinate-wise
Gaussian endpoint conditionals, as specified in
Appendix~\ref{app:lem-ntd}. At a fixed node $i$, let
$\varepsilon^{\mathrm{lower}}_{\mathrm{init},i}$ be the average
coordinate-wise optimal conditional RMSE for endpoint prediction, and
$\varepsilon^{\mathrm{lower}}_{\mathrm{NTD},i}$ the corresponding
optimal MAE for next-node prediction. Then
\begin{equation}
\varepsilon^{\mathrm{lower}}_{\mathrm{NTD},i}
=\lambda_i\sqrt{\tfrac{2}{\pi}}\,
\varepsilon^{\mathrm{lower}}_{\mathrm{init},i}
<\varepsilon^{\mathrm{lower}}_{\mathrm{init},i},
\label{eq:ntd-bound}
\end{equation}
where $\lambda_i=1-t_{i+1}/t_i\in[0,1]$ and the strict inequality
requires $\varepsilon^{\mathrm{lower}}_{\mathrm{init},i}>0$.
Averaging over nodes also preserves the strict comparison whenever
the average endpoint floor is positive.
\end{lemma}

\begin{proof}[Proof sketch]
For a scalar Gaussian conditional with standard deviation $s$,
the optimal conditional RMSE is $s$ and the optimal MAE is
$s\sqrt{2/\pi}$. The affine target relation scales the conditional
standard deviation by $\lambda_i$, yielding \cref{eq:ntd-bound}.
This connects next-node target contraction to the coordinate-wise
regression-error floors. See Appendix~\ref{app:lem-ntd} for the full
derivation and interpretation.
\end{proof}

\begin{figure}[t]
\centering
\begin{minipage}[t]{0.49\linewidth}
\begin{algorithm}[H]
\caption{One generator update with NTD + Chunk-Graded Noising.}
\label{alg:ode-init}
\small  
\begin{algorithmic}[1]
\Require Student $\Phi_{\theta}$, teacher trajectories full \newline \
$\{z_{t_k,c}^{\mathrm{tea}}\}_{k=0,\,c=1}^{T,\,C}$, student nodes \newline $\{\tau_i\}_{i=1}^{N+1}$ ($\tau_{N+1}=0$), chunk count $C$, learning rate $\eta$.

  \State Sample $s_c\sim\mathcal{U}\{1,\dots,N\}$ for $c=1,\dots,C$
  \State $(s_1,\dots,s_C)\leftarrow \mathrm{sort}_{\downarrow}(s_1,\dots,s_C)$
  \State Init start-point:
  \Statex \hspace{\algorithmicindent}$x_{\tau_{s_c},c}\leftarrow z_{\tau_{s_c},c}^{\mathrm{tea}}$ for $c=1,\dots,C$
  \State $\{x_{\tau_{s_c+1},c}\}_{c=1}^{C}\leftarrow \Phi_{\theta}\!\left(\{x_{\tau_{s_c},c}\}_{c=1}^{C},\{\tau_{s_c}\}_{c=1}^{C}\right)$
  \State $\mathcal{L}_{\mathrm{NTD}}\leftarrow \sum_{c=1}^{C}\big\|x_{\tau_{s_c+1},c}-z_{\tau_{s_c+1},c}^{\mathrm{tea}}\big\|_{1}$
  \State $\theta\leftarrow \theta-\eta\,\nabla_{\theta}\mathcal{L}_{\mathrm{NTD}}$

\end{algorithmic}
\vspace*{5pt}
\end{algorithm}
\end{minipage}
\vspace{-3mm}
\hfill
\begin{minipage}[t]{0.49\linewidth}
\begin{algorithm}[H]
\caption{OP-DMD Loss}
\label{alg:op-dmd}
\small
\begin{algorithmic}[1]
\Require Student $\Phi_{\theta}$, real score $s_{\mathrm{real}}$, fake score $s_{\mathrm{fake}}$, range $[t_{\min},t_{\max}]$, weights $\lambda_{\perp},\lambda_{\parallel}$ .
\State $\hat{x}_0^{\mathrm{stu}} \leftarrow \textsc{Generator}(\Phi_{\theta})$
\State Sample $t\sim\mathcal{U}[t_{\min},t_{\max}]$, $\epsilon\sim\mathcal{N}(0,I)$
\State $x_t \leftarrow (1-t)\,\hat{x}_0^{\mathrm{stu}} + t\,\epsilon$
\State $\hat{x}_0^{\mathrm{real}} \leftarrow s_{\mathrm{real}}(x_t,t)$;\ \  $\hat{x}_0^{\mathrm{fake}} \leftarrow s_{\mathrm{fake}}(x_t,t)$
\State $r \leftarrow \hat{x}_0^{\mathrm{real}}-\hat{x}_0^{\mathrm{fake}}$;\ \ $\hat{\mu}_{\mathrm{DMD}} \leftarrow \mathrm{sg}(r/\|r\|)$
\State $\delta \leftarrow \hat{x}_0^{\mathrm{stu}}-\hat{x}_0^{\mathrm{real}}$;\ \ $\alpha \leftarrow \langle\hat{\mu}_{\mathrm{DMD}},\delta\rangle$
\State $\delta_{\parallel} \leftarrow \alpha\,\hat{\mu}_{\mathrm{DMD}}$;\ \ $\delta_{\perp} \leftarrow \delta-\delta_{\parallel}$
\State $\mathcal{L} \leftarrow \mathcal{L}_{\mathrm{DMD}} + \lambda_{\perp}\mathcal{R}(\delta_{\perp}) + \lambda_{\parallel}\mathcal{R}(\alpha)$
\State $\theta \leftarrow \theta - \eta\,\nabla_{\theta}\mathcal{L}$
\State Update $s_{\mathrm{fake}}$
\end{algorithmic}
\vspace*{0pt}
\end{algorithm}
\end{minipage}
\end{figure}
\vspace{-3mm}

\subsubsection{Chunk-Graded Noising: compressing the noise mismatch
\texorpdfstring{$N$}{N}}
\label{sec:chunk-noising}

Existing distillation frameworks like Self-Forcing sample the per-chunk noise level
independently during training: in the trajectory-distillation stage the starting
node timestep $t_c$ is obtained at random, and in the distribution-matching stage
the exit timestep $t^c_{\mathrm{exit}}$ is also sampled at random from
$\{t_1,\ldots,t_K\}$. The induced training-time noise level distribution
$\sigma_{\mathrm{tr}}(c)$ can therefore deviate from the inference-time level
$\sigma_{\mathrm{inf}}(c)$ as context errors accumulate across chunks. This
deviation $|\sigma_{\mathrm{inf}}(c)-\sigma_{\mathrm{tr}}(c)|$ contributes to
the mismatch term $N$ in \cref{eq:bound-split}.

We compress this mismatch by imposing two constraints on per-chunk noise sampling. (1) \textbf{Clean Start.} The first chunk ($c=1$) is generated without contamination from previous chunks, so its inference-time noise level is identically zero. We mirror this at training by enforcing $\sigma_{\mathrm{tr}}(1)\!=\!0$ in both stages, removing mismatch on the first chunk. Since the geometric weights $w_c$ decrease monotonically with $c$, eliminating mismatch at early chunks substantially reduces the cumulative mismatch error $N$.
(2) \textbf{Sorted Noise.} For subsequent chunks ($c\!>\!1$), we sample per-chunk noise level under a non-decreasing constraint: (i) during NTD ODE stage, the sampled startpoint timesteps satisfy $t_{c}\!\ge\!t_{c-1}$;
(ii) during distribution matching stage, chunk-wise exit timestep satisfies
$t^{c}_{\mathrm{exit}}\!\ge\!t^{c-1}_{\mathrm{exit}}$. Both constraints induce a monotonically non-decreasing $\sigma_{\mathrm{tr}}(c)$, mirroring the natural error accumulation along the chunk axis at inference.

\Needspace{10\baselineskip}
\begin{lemma}[Reduction of expected noise mismatch]
\label{lem:noise-mismatch}
Let $N_{\mathrm{ours}}$ and $N_{\mathrm{SF}}$ denote the cumulative
mismatch of CGN and the baseline, respectively. Under the effective-noise
model and conditions in Appendix~\ref{app:proof-noise},
\begin{equation}
\mathbb E[N_{\mathrm{ours}}]
<
\mathbb E[N_{\mathrm{SF}}].
\label{eq:noise-mismatch}
\end{equation}
\end{lemma}
The derivation is given in Appendix~\ref{app:proof-noise}.
We integrate the Chunk-Graded Noising schedule with NTD to formulate the ODE initialization regression training. The comprehensive training algorithm is presented in \Cref{alg:ode-init}. By training the model to be robust to noisy context, we shrink the train-inference distribution shift. As a result, at inference time we can \seedtextit{completely skip} the post-chunk KV-cache re-computation step, which dramatically reduces the computation cost and boosts actual generation speed.
\begin{figure}[t]
  \centering
  \includegraphics[width=0.94\linewidth]{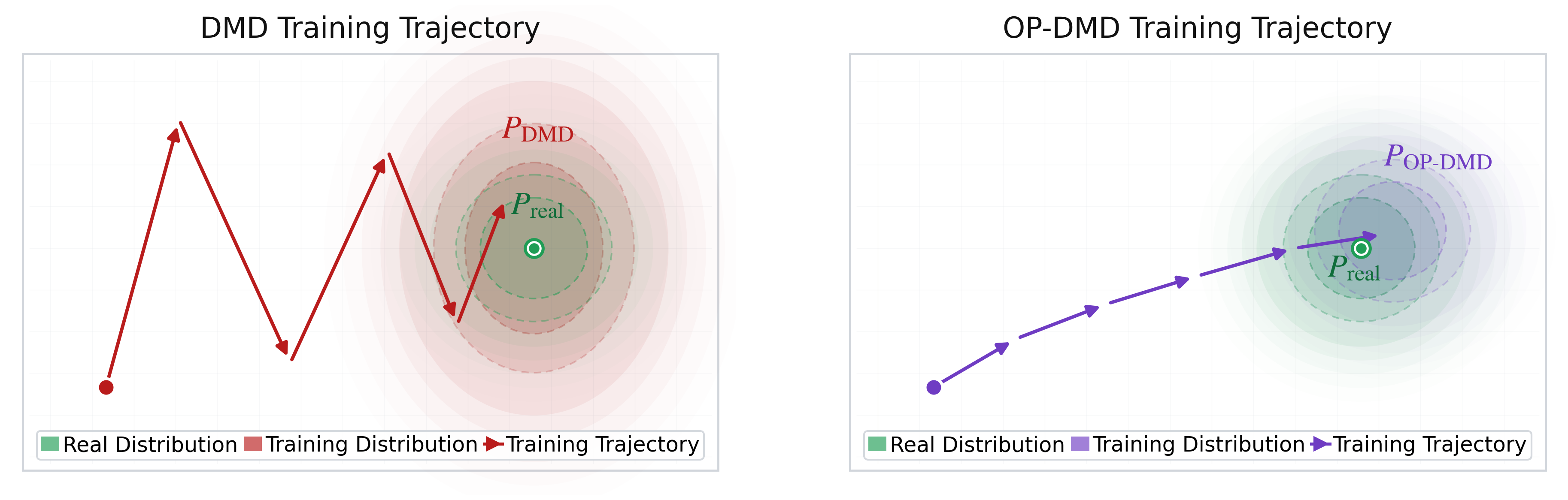}
  \caption{Training-trajectory comparison in a 2D latent toy.
  \textbf{Left:} Standard DMD with training oscillation and unconstrained subspace shift between $P_{\mathrm{DMD}}$ and $P_{\mathrm{real}}$. \textbf{Right:} OP-DMD with
  orthogonal penalty constraints, yields a stable training trajectory toward $P_{\mathrm{real}}$.}
  \vspace{-4mm}
  \label{fig:op-dmd-trajectory}
\end{figure}
\subsubsection{Orthogonal-Parallel DMD (OP-DMD): narrowing the realized-to-theoretical gap}
\label{sec:op-dmd}

The standard DMD loss updates the student generator with the gradient
signal $r=\hat{x}^{\mathrm{real}}_0-\hat{x}^{\mathrm{fake}}_0$, which provides
KL supervision along the DMD update direction $\hat{\mu}_{\mathrm{DMD}}$ in the latent space. However, the orthogonal subspace receives no explicit supervision, which induces potential shift in the training distribution and oscillation in the optimization trajectory. We refer to this effect as \seedtextit{off-axis drift}.

Based on this observation, we propose \seedtextit{Orthogonal-Parallel DMD} (OP-DMD), which decomposes the student-teacher residual along $\hat{\mu}_{\mathrm{DMD}}$ and applies separate supervision on the parallel and orthogonal components. As shown in \Cref{alg:op-dmd}, $\hat{\mu}_{\mathrm{DMD}}=r/\|r\|$ and $\delta=\hat x^{\mathrm{student}}_{0}-\hat{x}^{\mathrm{real}}_0$, then OP-DMD decomposes $\delta$ along $\hat{\mu}_{\mathrm{DMD}}$: 
\begin{equation}
\delta=\delta_{\parallel}+\delta_{\perp},\quad
\delta_{\parallel}=\langle\hat{\mu}_{\mathrm{DMD}},\delta\rangle\,\hat{\mu}_{\mathrm{DMD}},\quad
\delta_{\perp}=\delta-\delta_{\parallel}, \quad
\langle \hat{\mu}_{\mathrm{DMD}},\delta_{\perp}\rangle\ = 0.
\label{eq:op-decomp}
\end{equation}
Geometrically, $\delta_{\parallel}$ measures how much $\hat x^{\mathrm{student}}_{0}$ deviates from the teacher prediction along the KL-improving direction, while $\delta_{\perp}$ captures drift in the subspace unconstrained by DMD. The OP-DMD objective augments the standard DMD loss with pointwise
penalties on both components:
\begin{equation}
\mathcal{L}_{\mathrm{OP\text{-}DMD}}
=\mathcal{L}_{\mathrm{DMD}}
+\lambda_{\perp}\,\mathcal{R}(\delta_{\perp})
+\lambda_{\parallel}\,\mathcal{R}(\langle\hat{\mu}_{\mathrm{DMD}},\delta\rangle),
\quad \lambda_{\parallel}\!\ll\!\lambda_{\perp},
\label{eq:op-dmd-loss}
\end{equation}
where $\mathcal{R}(\cdot)$ denotes an $\ell_1$ regression loss. As shown in \cref{fig:op-dmd-trajectory} (right), the orthogonal term shrinks $\delta_{\perp}$ in the DMD-unsupervised subspace and constrains the orthogonal directions that would otherwise drift, suppressing off-axis drift. The parallel term lightly penalizes the on-axis projection coefficient (with $\lambda_{\parallel} \ll \lambda_{\perp}$) and stabilizes training as $\|r\|\!\to\!0$. A full OP-DMD step is described in \Cref{alg:op-dmd}.
\section{Experiments}
\label{sec:experiments}
\subsection{Experimental setup}
\label{sec:exp-setup}

\begin{figure}[t!]
  \centering
  \makebox[\linewidth][c]{\includegraphics[width=1.03\linewidth]{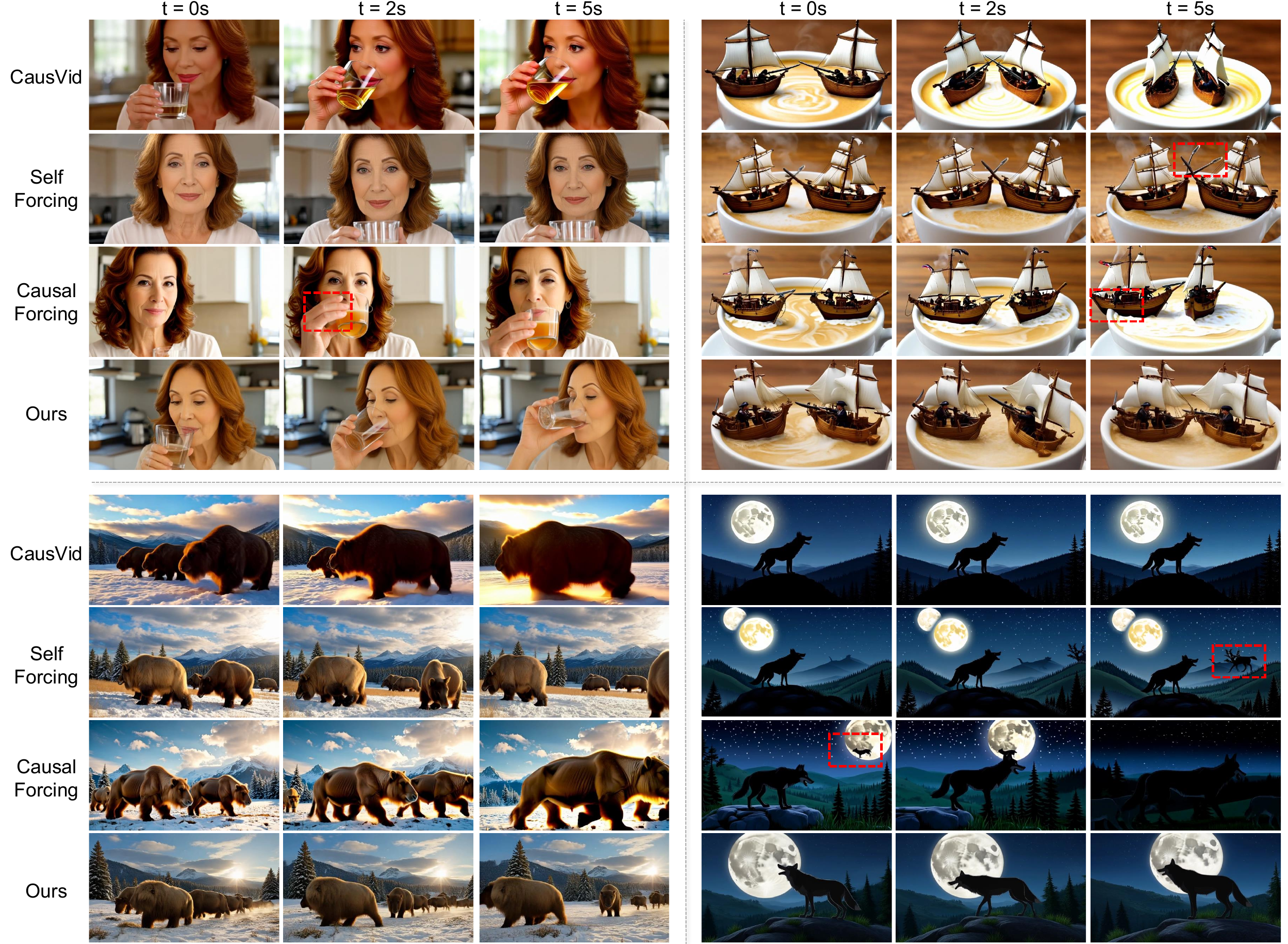}}
  \caption{Qualitative comparison against three distillation-based AR baselines (CausVid, Self-Forcing, Causal-Forcing). Each row shows key frames sampled from 5-second generation.}
  \vspace{-4mm}
  \label{fig:qualitative}
\end{figure}

\paragraph{Implementation.}
We adopt Wan2.1-1.3B~\cite{wan2024wan} as our base model, which generates 5-second videos by default at $832\!\times\!480$ resolution. Following
CausVid~\cite{yin2024causvid}, we first fine-tune the model with
block-causal attention masking on 6K teacher ODE pairs, using prompts
sampled from MixKit. Our student model runs $4$ denoising steps per
chunk, and each chunk consists of $3$ frames. The OP-DMD loss coefficients $\lambda_{\parallel}$ and $ \lambda_{\perp}$ are set to $0.0001$ and $0.01$, respectively. Additional implementation details including hyper-parameter settings are provided in \Cref{app:experiment}. \paragraph{Evaluation \& baselines.}
We compare our method against representative open-source video generation models. 
(1) \seedtextit{Bidirectional video diffusion models} such as Wan2.1-1.3B~\cite{wan2024wan} and LTX-Video-1.9B~\cite{hacohen2024ltxvideo}.
(2) \seedtextit{Autoregressive video diffusion models} like SkyReels-V2-DF-540P~\cite{skyreels2025}, MAGI-1 (base)~\cite{magi12025}. In particular, we compare with three distillation-based autoregressive baselines: CausVid~\cite{yin2024causvid}, Self-Forcing~\cite{huang2024selfforcing}, and the recent work Causal-Forcing~\cite{causalforcing2025}. We evaluate all methods with a 5-second generation horizon on VBench~\cite{huang2023vbench} benchmark to assess the generation quality. Following CausVid~\cite{yin2024causvid}, we use the first $128$ prompts from MovieGen to generate videos. We report speed test results both with and without \texttt{torch.compile} optimizations. FPS measures diffusion-backbone throughput, excluding VAE decoding.

\subsection{Experiment results}
\paragraph{Qualitative comparison.} \cref{fig:qualitative} shows key frames from $5$-second videos. CausVid exhibits notable chunk-level \seedtextit{error accumulation}, manifesting as severe color oversaturation and degraded image quality. Self-Forcing mitigates exposure bias through self-rollout and accumulates less error than CausVid, yet still suffers from artifacts. Causal-Forcing, while achieving stronger dynamic degree, is more prone to over-saturated colors and exhibits unrealistic spatial arrangements and structural artifacts. In contrast, our method preserves stable subject identity and image quality throughout the sequence. More visualization results and comparison are provided in \Cref{app:visualization}.

\begin{table*}[t!]
\centering
\caption{Quantitative comparison on VBench using the first 128 MovieGen prompts under a 5-second generation horizon. Throughput (FPS) is measured on a single NVIDIA H100 GPU.}
\label{tab:moviegen}
\small
\setlength{\tabcolsep}{2pt}
\renewcommand{\arraystretch}{1.05}
\resizebox{\textwidth}{!}{%
\begin{tabular}{l cc cccccc c |c}
\toprule
Method
 & \makecell{FPS w/ \\compile}
 & \makecell{FPS w/o \\compile}
 & \makecell{Subject\\Cons.}
 & \makecell{Background\\Cons.}
 & \makecell{Dynamic\\Degree}
 & \makecell{Image\\Quality}
 & \makecell{Motion\\Smooth.}
 & \makecell{Aesthetic\\Quality}
 & \makecell{Temporal\\Flicker.}
 & \textbf{\makecell{Quality\\Score}} \\
\midrule
\rowcolor{rowgray}\multicolumn{11}{l}{\seedtextit{Bidirectional Video Diffusion Models}} \\
Wan2.1            & 1.0  & 0.8  & 94.16 & 96.21 & 57.03 & 66.64 & 98.19 & 64.71 & 96.88 & 82.17 \\
LTX-Video         & 10.8 & 9.0  & 92.69 & 95.04 & 41.41 & 63.07 & 98.92 & 55.22 & 97.49 & 79.10 \\
\midrule
\rowcolor{rowgray}\multicolumn{11}{l}{\seedtextit{Autoregressive Video Diffusion Models}} \\
SkyReels-V2       & 0.6  & 0.5  & 95.87 & 96.86 & 32.81 & 62.51 & 98.72 & 63.45 & 97.91 & 80.63 \\
MAGI-1            & 0.3  & 0.2  & 96.29 & 96.50 & 35.94 & 61.71 & 99.22 & 64.76 & 98.39 & 81.42 \\
\midrule
\rowcolor{rowgray}\multicolumn{11}{l}{\seedtextit{Distilled Autoregressive Video Models}} \\
CausVid           & 20.3 & 16.9 & 94.65 & 94.51 & 45.31 & 68.00 & 98.27 & 62.10 & 97.06 & 80.93 \\
Self-Forcing      & 20.3 & 16.9 & 93.53 & 94.62 & 63.28 & 69.04 & 98.27 & 63.36 & 96.70 & 82.33 \\
Causal-Forcing    & 20.3 & 16.9 & 93.99 & 94.93 & 78.91 & 70.73 & 97.43 & 64.03 & 95.02 & 82.91 \\
\midrule
\textbf{Ours}     & \textbf{26.8} & \textbf{22.2} & 94.25 & 94.66 & 82.03 & 69.25 & 98.26 & 63.46 & 95.82 & \textbf{83.59} \\
\bottomrule
\end{tabular}
\vspace{-10mm}
}
\end{table*}

\paragraph{Quantitative comparison.} For chunk-wise autoregressive baselines, though they adopt 4-step generation, they actually require $5$-NFE to finish single chunk generation, as an extra forward pass is required after each chunk's $4$ denoising steps to construct a clean KV-cache (higher quality but higher latency). In contrast, our method uses only $4$ NFEs per chunk, eliminating the extra forward pass. \Cref{tab:moviegen} shows that our method achieves the best quality score of \textbf{83.59} and the highest
throughput, reaching \textbf{26.8}/\textbf{22.2} FPS (with/without
compilation). Notably, with only $4$ NFE, our method still consistently outperforms all distilled autoregressive baselines with $5$-NFE. For example, surpassing Self-Forcing and Causal-Forcing by $+1.26$ and $+0.68$ in Quality Score, respectively, delivering impressive quality-efficiency trade-off among existing approaches.

\Needspace{13\baselineskip}
\paragraph{Human preference.}
\begin{wraptable}{r}{0.56\linewidth}
  \vspace{-2pt}
  \centering
  \small
  \setlength{\tabcolsep}{5pt}
  \renewcommand{\arraystretch}{1.1}
  \caption{Pairwise human preference on $100$ VBench prompts ($500$ votes per comparison).}
  \label{tab:human-preference}
  \begin{tabular}{lcc}
    \toprule
    Comparison & TriTD ($4$ NFE) & Baseline ($5$ NFE) \\
    \midrule
    vs. CausVid        & \textbf{438 (87.6\%)} & 62 (12.4\%) \\
    vs. Self-Forcing   & \textbf{323 (64.6\%)} & 177 (35.4\%) \\
    vs. Causal-Forcing & \textbf{354 (70.8\%)} & 146 (29.2\%) \\
    \bottomrule
  \end{tabular}
  \vspace{-8pt}
\end{wraptable}
To complement the automatic metrics, we conducted a blinded pairwise evaluation on $100$ randomly sampled VBench prompts with five evaluators. As shown in \Cref{tab:human-preference}, TriTD receives the majority of preference votes against all three distilled AR baselines despite requiring fewer effective NFEs. The complete evaluation protocol and absolute-rating results are provided in \Cref{app:human-eval}.

\Needspace{15\baselineskip}
\subsection{Ablation and analysis}
\label{sec:ablation}

\paragraph{Ablation study.} 
\begin{wraptable}{r}{0.52\linewidth}
  \vspace{-12pt}   
  \centering
  \small
  \setlength{\tabcolsep}{4pt}
  \caption{Ablation of the three proposed components.} 
  \label{tab:ablation}
  \begin{tabular}{clcccc}
    \toprule
    \# & Variant & NTD & CGN & OP-DMD & Quality\,$\uparrow$ \\
    \midrule
    M0 & Baseline                  & \xmark & \xmark & \xmark & 81.82 \\
    M1 & \;+ NTD init              & \cmark & \xmark & \xmark & 82.07 \\
    M2 & \;+ CGN  & \cmark & \cmark & \xmark & 82.10 \\
    M3 & \;+ OP-DMD & \cmark & \cmark & \cmark & \textbf{83.59} \\
    \bottomrule
  \end{tabular}
  \vspace{-10pt}
\end{wraptable}
We further ablate the three components on top of the Self-Forcing-like baseline and report the Quality Score on MovieGen prompts in \Cref{tab:ablation}. Replacing the endpoint regression at the ODE-init stage with our NTD L1 trajectory
distillation (M1) lifts the baseline from $81.82$ to $82.07$, showing the benefit of next-node supervision for ODE initialization. Applying chunk-graded noising schedule with M1 strategy (M2) further pushes the score to $82.10$ and,
more importantly, reduces the train-inference KV-cache noise mismatch, allowing the student to skip the per-chunk KV-cache refresh at inference and thereby enabling the generation speed-up reported in \Cref{tab:moviegen}. Finally, replacing the standard DMD objective in M2 strategy with OP-DMD (M3) yields the largest jump to \textbf{83.59}. By imposing a pointwise constraint on the DMD-orthogonal subspace, OP-DMD suppresses off-axis drift and complements the trajectory and context-distribution objectives motivated by our analysis. 

\Needspace{9\baselineskip}
\paragraph{Efficiency comparison \& breakdown.}
As shown in \Cref{fig:latency}, the VAE decoding time remains nearly constant at $\approx 1.74$\,s across all settings, indicating that the gain is almost entirely due to the diffusion backbone. In the SF baseline, each chunk requires four denoising forward passes and one additional forward pass to refresh the KV-cache, resulting in $4.79$\,s diffusion cost and $6.53$\,s total latency per video. Applying \texttt{torch.compile} to SF alone reduces the diffusion cost to $3.98$\,s ($1.20{\times}$ core speedup) and the total latency to $5.73$\,s ($1.14{\times}$ end-to-end speedup).

\begin{wrapfigure}{r}{0.44\linewidth}
  \vspace{-4pt}
  \centering
  \includegraphics[width=\linewidth]{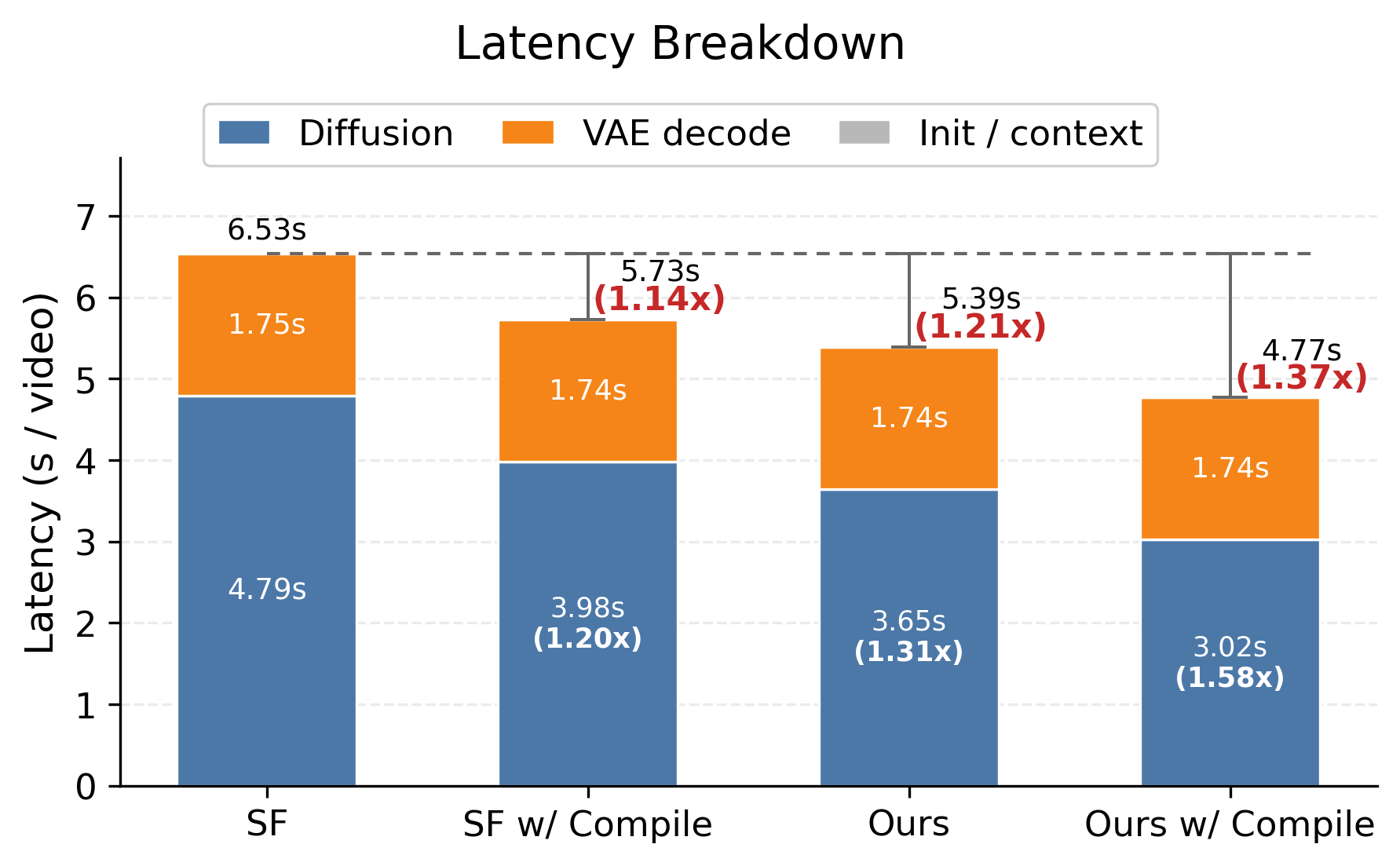}
  \captionsetup{width=\linewidth}
  \caption{\textbf{Latency breakdown for four settings.}}
  \label{fig:latency}
  \vspace{-8pt}
\end{wrapfigure}
Because the chunk-graded noising schedule trains the student to remain robust to noisy KV Cache context, our method can directly reuse the noisy KV Cache produced by the last denoising step at inference and skip the refresh pass entirely. This reduces the diffusion cost to $3.65$\,s ($1.31{\times}$ core speedup) and the end-to-end latency to $5.39$\,s ($1.21{\times}$ end-to-end speedup) without compilation. Combined with \texttt{torch.compile}, the diffusion cost further drops to $3.02$\,s ($1.58{\times}$ core speedup), yielding $4.77$\,s total latency per video ($1.37{\times}$ end-to-end speedup) and the $26.8$ FPS reported in \Cref{tab:moviegen}. Notably, our uncompiled variant already outperforms the compiled SF baseline ($5.39$\,s vs.\ $5.73$\,s), suggesting that the gain mainly comes from the algorithmic redesign of the KV-cache pipeline rather than compilation-level engineering.

\Needspace{13\baselineskip}
\paragraph{Long-horizon generalization.}
\begin{wraptable}{r}{0.48\linewidth}
  \vspace{-12pt}
  \centering
  \small
  \renewcommand{\arraystretch}{1.12}
  \captionsetup{width=\linewidth}
  \caption{\textbf{Long-horizon generation over 60 seconds \mbox{using} MovieGen prompts.}}
  \label{tab:long-video}
  \begin{tabular*}{\linewidth}{@{\extracolsep{\fill}}lccc@{}}
    \toprule
    Method & Dyn.\,Deg. & Motion\,Sm. & Quality \\
    \midrule
    Self-Forcing    & 33.26 & 98.11 & 80.35 \\
    Causal-Forcing  & 46.30 & 95.97 & 77.01 \\
    \textbf{Ours}   & \textbf{85.23} & \textbf{97.66} & \textbf{82.64} \\
    \bottomrule
  \end{tabular*}
  \vspace{-8pt}
\end{wraptable}
Although trained on $5$-second clips, TriTD outperforms baselines when the inference horizon is extended to 60 seconds. As shown in \mbox{\Cref{tab:long-video}}, the two baselines exhibit qualitatively distinct failure modes: Self-Forcing falls into near-static rollouts, behaving as low dynamic degree and high motion smoothness in return. Causal-Forcing exhibits rapid visual breakdown, of which both image quality and motion smoothness significantly degrade. In contrast, TriTD retains the highest quality score at 60 seconds, supporting its stability over longer generation horizons. Full-dimension VBench metrics are provided in Appendix~\ref{app:long-video}.

\section{Conclusion and limitations}
\label{sec:conclusion}

We presented \textbf{TriTD}, a tri-partite trajectory-distribution co-optimized distillation framework for real-time autoregressive video diffusion. Starting from a decomposition of the per-chunk inference error into base model error, train-inference noise mismatch, and upstream error propagation, we derive a closed-form upper bound on cumulative error that guides three complementary modules: NTD uses next-node trajectory supervision to improve ODE initialization; CGN aligns training context noise with inference and enables KV-cache reuse; and OP-DMD constrains the DMD-orthogonal subspace to stabilize optimization. As a result, our $4$-NFE generator attains the highest VBench quality score among existing AR models and reaches $26.8$ FPS on a single H100, a $1.32{\times}$ speedup over Self-Forcing, pushing AR video generation into the real-time interactive regime.

\paragraph{Limitations.} TriTD targets chunk-wise video diffusion distillation and is designed to establish a stable and efficient training paradigm rather than to address minute-scale long-video synthesis, for which dedicated mechanisms (e.g., Deep Forcing~\cite{yi2025deep}) are likely complementary to our framework and represent a promising direction for future work. In addition, our training data consists of $6$K prompts from MixKit, substantially smaller than the $16$K setting used in Self-Forcing and related baselines. Scaling up training data is therefore a natural direction for further improving TriTD. Following common practice in autoregressive video distillation~\cite{huang2024selfforcing,causalforcing2025}, all experiments are conducted on Wan2.1. Verifying TriTD on additional backbones and exploring its transferability to other generative tasks (e.g., image and audio synthesis) are also left for future investigation.


{
\small
\bibliographystyle{unsrtnat}
\bibliography{references}
}


\clearpage
\beginappendix

\section{Technical appendices and supplementary material}
\label{app:main}

This appendix contains: (1) the detailed proofs of the lemmas and corollaries stated in Section~\ref{sec:error-bound}--\ref{sec:chunk-noising} of the main text, (2) Detailed experimental settings, (3) supplementary experimental results and (4) additional visualization results and comparison with representative baselines. 

\subsection{Proof of \Cref{cor:cumulative}: Total Accumulated Error Upper Bound}
\label{app:proof-total}

\paragraph{Step 1: Quantifying each error term.}
\begin{itemize}
\item \textbf{Base error.} Assume the residual under matched, error-free context is uniformly bounded by $\varepsilon$. Here $\varepsilon$ denotes the residual upper bound; Appendix~\ref{app:lem-ntd} separately analyzes optimal regression-error floors.
\item \textbf{Distribution-shift error.} At inference time the model receives a KV-cache $\mathcal{K}_{<c}^{\text{inf}}$ whose noise level is $\sigma_{\text{inf}}(c)$, whereas at training time it has been trained on caches of noise level $\sigma_{\text{tr}}(c)$. The resulting mismatch is modeled as a term proportional to the absolute noise-level difference, namely $\kappa \cdot |\sigma_{\text{inf}}(c) - \sigma_{\text{tr}}(c)|$, where $\kappa>0$ reflects the sensitivity of the model to the noise-level change.
\item \textbf{Propagation error.} The accumulated error $e_{c-1}$ from the previous $c-1$ chunks affects the generation of chunk $c$ via the KV-cache. Under a \seedtextit{Lipschitz-continuity assumption}, this effect is bounded: the denoising operator $f$ processes an input that contains an upstream error, and so this error is scaled by $L$ to yield $L \cdot e_{c-1}$.
\end{itemize}

\paragraph{Step 2: Combining the three terms.}
Summing the three sources, we obtain the single-step upper bound
\begin{equation}
e_c \le \varepsilon + \kappa \cdot |\sigma_{\text{inf}}(c) - \sigma_{\text{tr}}(c)| + L \cdot e_{c-1}.
\end{equation}
This follows from the residual and context-perturbation bounds in Lemma~\ref{lem:single-step}.

\paragraph{Step 3: Recursive expansion and summation.}
\seedtextit{Recursive expansion of $e_c$.} Assuming $e_0=0$ and that the upstream scaling is uniform with constant $L$, and writing $\varepsilon_c := \varepsilon + \kappa \cdot |\sigma_{\text{inf}}(c) - \sigma_{\text{tr}}(c)|$, we have
\begin{align}
e_C &\le \varepsilon_C + L \cdot e_{C-1} \\
    &\le \varepsilon_C + L \cdot \varepsilon_{C-1} + L^2 \cdot e_{C-2} \\
    &\ \ \vdots \notag \\
e_c &\le \sum_{i=1}^{c} L^{c-i} \cdot \varepsilon_i.
\end{align}
\seedtextit{Summing over all chunks.}
\begin{equation}
E_C = \sum_{c=1}^{C} e_c \le \sum_{c=1}^{C} \sum_{i=1}^{c} L^{c-i} \varepsilon_i.
\end{equation}
\seedtextit{Exchanging the order of summation.} A fixed term $\varepsilon_i$ appears in every chunk $c \ge i$, each time scaled by some power of $L$, so
\begin{equation}
E_C \le \sum_{i=1}^{C} \Biggl( \sum_{c=i}^{C} L^{c-i} \Biggr)\, \varepsilon_i.
\end{equation}
\seedtextit{Geometric series.}
\begin{equation}
\sum_{c=i}^{C} L^{c-i} \;=\; L^0 + L^1 + \cdots + L^{C-i} \;=\; \frac{1-L^{C-i+1}}{1-L}.
\end{equation}
\seedtextit{Final form.} Substituting the geometric-series identity and renaming $i$ as $c$, we arrive at
\begin{equation*}
E_C \le \sum_{c=1}^{C} \left( \frac{1-L^{C-c+1}}{1-L} \right) \left( \varepsilon + \kappa \cdot |\sigma_{\text{inf}}(c) - \sigma_{\text{tr}}(c)| \right). \qquad\square
\end{equation*}

\paragraph{Remark on the Lipschitz constant $L$.}
The derivation above implicitly assumes $L\neq 1$ when applying the
geometric-series identity. For completeness, the closed form of $w_c$
covers all $L\ge 0$:
\begin{equation}
w_c \;=\;
\begin{cases}
\dfrac{1-L^{C-c+1}}{1-L}, & L\in[0,1)\cup(1,\infty),\\[4pt]
C-c+1, & L=1\ \text{(by L'H\^opital)}.
\end{cases}
\end{equation}
For any \seedtextit{finite} chunk count $C$, $w_c$ is therefore a finite
constant regardless of the value of $L$, and Corollary~\ref{cor:cumulative}
holds without modification. The Lipschitz constant $L$ only governs the
asymptotic regime $C\!\to\!\infty$: $w_c$ converges to $1/(1-L)$ when
$L<1$, grows linearly when $L=1$, and grows exponentially when $L>1$.
For fixed $L$ and $\kappa$, reducing either modeled contribution
reduces the bound for any finite $C$. This identifies base-model fitting
and context-noise alignment as complementary optimization directions.
The NTD analysis below examines the regression-error floors associated
with the first direction.

\subsection{Proof of \Cref{lem:ntd-bound}: NTD Regression-Error Floors}
\label{app:lem-ntd}

\paragraph{Training objectives.}
The endpoint baseline is trained with MSE, as in
\cref{eq:endpoint-l2}, while NTD uses an $L_1$ loss against the
next saved teacher ODE node, as in \cref{eq:ntd-l1}.
Teacher trajectories are collected before training, and their saved
next nodes directly supervise NTD. The following analysis adopts
idealized local-affine and conditional-Gaussian assumptions to motivate
this design.

\paragraph{Target model and assumptions.}
Fix a chunk $c$ and node $i$. Let $X$ contain the information
available to the causal student, including the current noisy
chunk $x_{t_i}$ and its visible context. Let
$Y=x_0^\star\in\mathbb R^D$ and $Z=x_{t_{i+1}}^\star\in\mathbb R^D$
be the endpoint and next-node targets. We assume
\begin{equation}
Z=\lambda_iY+b_i(X),\qquad
\lambda_i=1-\frac{t_{i+1}}{t_i}\in[0,1],\qquad
b_i(X)=(1-\lambda_i)x_{t_i},
\label{eq:ntd-affine-assumption}
\end{equation}
where $t_i>0$ and $0\le t_{i+1}\le t_i$.
This is a locally straight trajectory idealization. For each coordinate
$q$, assume a Gaussian conditional
\begin{equation}
Y_q\mid X,c,i\sim
\mathcal N\!\left(a_{c,i,q}(X),s_{c,i,q}^2(X)\right),
\qquad
\mathbb E[s_{c,i,q}^2(X)]<\infty.
\label{eq:ntd-gaussian-assumption}
\end{equation}
The coordinate-wise analysis allows correlated coordinates and
characterizes target uncertainty given the causal student's information.

\paragraph{Coordinate-wise regression-error floors.}
For a fixed $X,c,i,q$, define the optimal endpoint conditional RMSE
and next-node conditional MAE over scalar predictions $u$ and $v$:
\begin{align}
r_{\mathrm{init},c,i,q}(X)
&:=\inf_{u\in\mathbb R}
\sqrt{\mathbb E[(u-Y_q)^2\mid X,c,i]}
=s_{c,i,q}(X), \label{eq:endpoint-rmse-floor}\\
r_{\mathrm{NTD},c,i,q}(X)
&:=\inf_{v\in\mathbb R}
\mathbb E[|v-Z_q|\mid X,c,i].
\label{eq:next-mae-floor}
\end{align}
The conditional mean minimizes the endpoint MSE, and its square
root defines the RMSE measure above; the training objective itself
remains MSE. We average these scalar floors over inputs, chunks,
and coordinates:
\begin{align}
\varepsilon^{\mathrm{lower}}_{\mathrm{init},i}
&:=\mathbb E_{c,X}\!\left[
\frac1D\sum_{q=1}^{D}r_{\mathrm{init},c,i,q}(X)\right],\\
\varepsilon^{\mathrm{lower}}_{\mathrm{NTD},i}
&:=\mathbb E_{c,X}\!\left[
\frac1D\sum_{q=1}^{D}r_{\mathrm{NTD},c,i,q}(X)\right].
\end{align}
Thus the endpoint floor averages coordinate-wise optimal conditional
RMSEs, while the next-node floor averages coordinate-wise optimal MAEs.

\paragraph{Fixed-node comparison.}
By \cref{eq:ntd-affine-assumption,eq:ntd-gaussian-assumption},
$Z_q\mid X,c,i$ has standard deviation $\lambda_i s_{c,i,q}(X)$.
Its median and mean coincide, and the absolute deviation of a scalar
Gaussian about its mean is its standard deviation times
$\sqrt{2/\pi}$. Therefore
\begin{equation}
r_{\mathrm{NTD},c,i,q}(X)
=\lambda_i s_{c,i,q}(X)\sqrt{\frac{2}{\pi}}.
\end{equation}
Averaging at the fixed node gives
\begin{equation}
\varepsilon^{\mathrm{lower}}_{\mathrm{NTD},i}
=\lambda_i\sqrt{\frac{2}{\pi}}\,
\varepsilon^{\mathrm{lower}}_{\mathrm{init},i}.
\end{equation}
Since $\lambda_i\le1$ and $\sqrt{2/\pi}<1$, this proves
\cref{eq:ntd-bound} whenever
$\varepsilon^{\mathrm{lower}}_{\mathrm{init},i}>0$.
If that floor is zero, both sides are zero.

\paragraph{Averaging over nodes.}
Using the same node-sampling distribution for both measures, define
$\varepsilon^{\mathrm{lower}}_{\mathrm{init}}
:=\mathbb E_i[\varepsilon^{\mathrm{lower}}_{\mathrm{init},i}]$
and
$\varepsilon^{\mathrm{lower}}_{\mathrm{NTD}}
:=\mathbb E_i[\varepsilon^{\mathrm{lower}}_{\mathrm{NTD},i}]$.
Keeping the node-dependent $\lambda_i$ inside the expectation,
\begin{align}
\varepsilon^{\mathrm{lower}}_{\mathrm{NTD}}
&=\sqrt{\frac{2}{\pi}}\,
\mathbb E_i[\lambda_i\varepsilon^{\mathrm{lower}}_{\mathrm{init},i}]
\nonumber\\
&\le\sqrt{\frac{2}{\pi}}\,
\varepsilon^{\mathrm{lower}}_{\mathrm{init}}
<\varepsilon^{\mathrm{lower}}_{\mathrm{init}},
\end{align}
where the last inequality requires a positive average endpoint
floor. \qed

\paragraph{Interpretation.}
The factor $\lambda_i$ captures the contraction of conditional target
variability, while $\sqrt{2/\pi}$ relates optimal MAE to conditional
RMSE under the Gaussian model. Together, they explain the
regression-error comparison in \Cref{lem:ntd-bound} and motivate
next-node trajectory supervision.

\subsection{Proof of \Cref{lem:noise-mismatch}: Expected Noise-Mismatch Reduction}
\label{app:proof-noise}

By \cref{eq:bound-split}, write
\begin{equation}
N=\kappa\sum_{c=1}^{C}w_c\Delta_c,\qquad
\Delta_c=|\sigma_{\mathrm{inf}}(c)-\sigma_{\mathrm{tr}}(c)|,
\qquad w_c=\sum_{j=0}^{C-c}L^j>0.
\end{equation}
We compare expectations over the training-time sampling, holding
$L\ge0$, $\kappa>0$, and the weights fixed across methods.

\paragraph{Idealized additive effective-noise model.}
The context of chunk $c$ contains only the preceding $m=c-1$
chunks. The scalar $\sigma$ is an additive proxy for effective
context noise. Assume:
\begin{enumerate}
\item A rollout samples $C$ exit timesteps $T_1,\ldots,T_C$ from the
same marginal distribution. Let $U_j=\delta(T_j)$, where $\delta$
is non-decreasing, $U_j\ge d>0$ almost surely, and
$\mathbb E[U_j]=\mu<\infty$. The samples may be dependent.
For the uniform four-step setting, $\mu$ is the mean of the four
increment values.
\item In the SF baseline considered here, training retains the
sampled noisy cache without a clean refresh, whereas inference
refreshes the cache. The model assigns
$\sigma_{\mathrm{tr}}^{\mathrm{SF}}(c)=\sum_{j=1}^{m}U_j$
and $\sigma_{\mathrm{inf}}^{\mathrm{SF}}(c)=0$.
Zero denotes the refreshed-cache noise reference.
\item CGN orders the sampled increments as
$U_{(1)}\le\cdots\le U_{(C)}$ and retains the final-denoising cache
at inference. We model
$\sigma_{\mathrm{tr}}^{\mathrm{Ours}}(c)=\sum_{j=1}^{m}U_{(j)}$
and $\sigma_{\mathrm{inf}}^{\mathrm{Ours}}(c)=md$,
where $d$ is the fixed retained increment and does not exceed
any sampled increment.
\end{enumerate}
Both methods therefore have $\Delta_1=0$, consistent with Clean
Start.

\paragraph{Baseline mismatch.}
Nonnegativity and linearity of expectation give
\begin{equation}
\mathbb E[\Delta_c^{\mathrm{SF}}]
=\mathbb E\!\left[\sum_{j=1}^{m}U_j\right]=m\mu.
\label{eq:sf-expected-mismatch}
\end{equation}

\paragraph{Sorted-prefix bound.}
For every realization and $1\le m\le C$, the mean of the $m$
smallest increments is no larger than the mean of all $C$:
\begin{equation}
\sum_{j=1}^{m}U_{(j)}
\le\frac{m}{C}\sum_{j=1}^{C}U_j.
\end{equation}
Consequently,
\begin{equation}
\mathbb E\!\left[\sum_{j=1}^{m}U_{(j)}\right]\le m\mu.
\label{eq:sorted-prefix-bound}
\end{equation}
Thus sorting bounds the expected cumulative contribution of the
early chunks.

\paragraph{Our mismatch and cumulative comparison.}
Since $U_{(j)}\ge d$, for $m=c-1$,
\begin{align}
\mathbb E[\Delta_c^{\mathrm{Ours}}]
&=\mathbb E\!\left[\left|\sum_{j=1}^{m}U_{(j)}-md\right|\right]
\nonumber\\
&=\mathbb E\!\left[\sum_{j=1}^{m}U_{(j)}\right]-md
\le m(\mu-d).
\label{eq:ours-expected-mismatch}
\end{align}
The same formula holds for $m=0$ with empty sums. Thus
\begin{equation}
\mathbb E[\Delta_c^{\mathrm{SF}}]
-\mathbb E[\Delta_c^{\mathrm{Ours}}]\ge(c-1)d.
\end{equation}
For $C\ge2$, the common positive weights imply
\begin{equation}
\mathbb E[N_{\mathrm{SF}}]-\mathbb E[N_{\mathrm{Ours}}]
\ge\kappa d\sum_{c=2}^{C}w_c(c-1)>0,
\end{equation}
which proves \cref{eq:noise-mismatch}. \qed

\paragraph{Boundary cases.}
For $C=1$, both mismatches are zero. If the retained increment is
allowed to be $d=0$, the same argument yields a non-strict inequality.

\subsection{Detailed Experimental Settings}
\label{app:experiment}

\paragraph{Generation and evaluation protocol.}
Unless otherwise specified, all main-paper comparisons use a $5$-second
generation horizon at $832\!\times\!480$ resolution. Following
CausVid~\cite{yin2024causvid}, we use the first $128$ prompts from
MovieGen and evaluate the generated videos with VBench. The reported
quality score aggregates the visual-quality dimensions used in the
main comparison, including subject consistency, background
consistency, dynamic degree, image quality, motion smoothness,
aesthetic quality, and temporal flickering. The long-horizon
experiment extends the same autoregressive generation setting to
$60$ seconds; its detailed protocol and results are reported in
\Cref{app:long-video}.

\paragraph{Baselines.}
We compare TriTD with two bidirectional video diffusion models,
Wan2.1-1.3B \cite{wan2024wan} and LTX-Video-1.9B \cite{hacohen2024ltxvideo};
two autoregressive video diffusion models, SkyReels-V2-DF-540P \cite{skyreels2025}
and MAGI-1 (base) \cite{magi12025}; and three distilled autoregressive
baselines, CausVid \cite{yin2024causvid},
Self-Forcing \cite{huang2024selfforcing}, and
Causal-Forcing \cite{causalforcing2025}. This grouping separates the
quality--speed comparison with bidirectional generators from the
direct comparison with chunk-wise autoregressive methods.

\paragraph{Inference and efficiency measurement.}
Our student performs four denoising steps per chunk, with three frames
in each chunk. Throughput is measured on a single NVIDIA H100 GPU and
is reported both with and without \texttt{torch.compile}. The latency
breakdown uses end-to-end wall-clock time and separately reports the
diffusion-backbone and VAE-decoding costs; all speedups use the
uncompiled Self-Forcing configuration as the common reference unless
explicitly stated otherwise.

\subsection{Human Evaluation Protocol and Additional Results}
\label{app:human-eval}

\paragraph{Protocol.}
We randomly sample $100$ prompts from VBench, covering categories such as consistency, scene, action, and color, and invite five evaluators to participate. For every prompt, each evaluator views three pairs of $5$-second videos generated from the same prompt, comparing TriTD separately against CausVid, Self-Forcing, and Causal-Forcing. Method names are hidden, and the left--right positions are independently randomized. Each TriTD--baseline comparison therefore contains $100\times5=500$ pairwise votes. Evaluators select the overall better video by jointly considering visual quality and prompt alignment; the resulting pairwise preferences are reported in \Cref{tab:human-preference}.

\paragraph{Absolute ratings.}
Evaluators also rate each generated video from $0$ to $10$ along four dimensions: aesthetics, color harmony, motion/action rationality, and structural detail. The overall score is the equally weighted mean of these four dimensions. As shown in \Cref{tab:human-absolute}, TriTD obtains the highest overall score of $8.93$, compared with $7.40$ for CausVid, $8.73$ for Self-Forcing, and $8.15$ for Causal-Forcing. It achieves the highest or tied-highest score on every evaluated dimension.

\begin{table}[H]
  \centering
  \caption{Absolute human ratings on a $0$--$10$ scale. The overall score is the equally weighted mean of the four dimensions. Best scores are in \textbf{bold}.}
  \label{tab:human-absolute}
  \small
  \setlength{\tabcolsep}{10pt}
  \renewcommand{\arraystretch}{1.1}
  \begin{tabular}{lccccc}
    \toprule
    Method & Aesthetics & Color & Motion & Detail & Overall \\
    \midrule
    CausVid        & 7.3 & 7.6 & 6.8 & 7.9 & 7.40 \\
    Self-Forcing   & 8.6 & \textbf{9.2} & 8.4 & 8.7 & 8.73 \\
    Causal-Forcing & 7.9 & 8.1 & 8.1 & 8.5 & 8.15 \\
    \textbf{TriTD} & \textbf{8.8} & \textbf{9.2} & \textbf{8.5} & \textbf{9.2} & \textbf{8.93} \\
    \bottomrule
  \end{tabular}
\end{table}

\subsection{Implementation Details}
\label{app:impl}

We build on the Wan2.1-T2V $1.3$B and $14$B models as our base
backbone and real score teacher model. Both are originally bidirectional video diffusion models. Our training pipeline consists of three stages.

\paragraph{Stage 1: Teacher ODE trajectory collection.}
To construct the supervision signal for NTD trajectory distillation,
we adopt the same ODE trajectory preparation protocol as the official
CausVid repository. Concretely, we use $6{,}000$ prompts from the
MixKit dataset and, for each prompt, run the bidirectional teacher
model for $48$ denoising steps with a time-shift of $5.0$ to obtain
a full denoising trajectory. For every prompt, we cache the complete
$49$-step trajectory (including $x_0$), including the predicted clean endpoint, so that
any intermediate teacher node can serve as a supervision target. For
the $4$-step student, we select
$\{1000.0,\,923.913,\,750.0,\,500.0\}$ out of the $48$ teacher time
steps as the student inference timesteps.

\paragraph{Stage 2: ODE regression initialization.}
We train the student with the
proposed NTD L1 trajectory regression loss under the chunk-graded
noising schedule (Clean-Start + Sorted-Noise). This stage runs for
$2{,}700$ iterations with a total batch size of $64$.

\paragraph{Stage 3: OP-DMD distillation.}
We then switch to the OP-DMD objective and perform self-rollout
training on the student. For each rollout, the exit-flag indices of all
chunks are first sampled uniformly at random and then sorted in
descending order, so that the KV-cache noise level across chunks is
monotonically non-decreasing. This stage runs for $700$ iterations with a total batch size of $64$, taking approximately two days on the $8\times$ A100 node. We adopt a sliding attention window of $12$ frames at this stage. \paragraph{Hyperparameters.}The two training stages use different optimizer settings, summarized in \Cref{tab:hparams}. Other hyperparameters are shared across
stages: teacher CFG guidance scale $3.0$, total batch size $64$, $3$ frames per chunk.
\begin{table}[h]
\centering
\small
\setlength{\tabcolsep}{8pt}
\caption{Optimizer configuration of the two training stages. Stage~3
maintains separate optimizers for the generator and the critic.}
\label{tab:hparams}
\begin{tabular}{lccc}
\toprule
Stage & Learning rate & $\beta_1$ & $\beta_2$ \\
\midrule
Stage 2: ODE regression          & $5{\times}10^{-6}$ & 0.9 & 0.999 \\
Stage 3: OP-DMD (generator)      & $2{\times}10^{-6}$ & 0.0 & 0.999 \\
Stage 3: OP-DMD (critic)         & $4{\times}10^{-7}$ & 0.0 & 0.999 \\
\bottomrule
\end{tabular}
\end{table}

\section{Long-Horizon Generalization}
\label{app:long-video}

\Cref{tab:long-video-full} reports the full 7-dimension VBench
metrics for the three chunk-wise distillation methods compared in
\Cref{sec:ablation}, evaluated on the first 128 prompts from MovieGen. Consistent with the two failure
modes described in the main text, Self-Forcing's Dynamic Degree collapses to $33.26$ (from $63.28$ at $5$s), while its remaining static frames retain a high Motion Smoothness of $98.11$, which is a signature of the near-static collapse mode. Causal-Forcing instead
suffers from simultaneous degradation in Motion Smoothness
($95.97$) and Image Quality ($64.26$), with its Quality Score
dropping to $77.01$. TriTD retains a
Dynamic Degree of $85.23$ and achieves the highest Quality Score
of $82.64$, surpassing the two baselines by $+2.29$ and $+5.63$,
respectively.

\begin{table}[H]
\centering
\caption{\textbf{Comparison on $60$-second generation}. Best score per column is
\textbf{bold}.}
\label{tab:long-video-full}
\small
\setlength{\tabcolsep}{3.5pt}
\renewcommand{\arraystretch}{1.1}
\begin{tabular}{l ccccccc c}
\toprule
Method
 & \makecell{Subject\\Cons.}
 & \makecell{Background\\Cons.}
 & \makecell{Dynamic\\Degree}
 & \makecell{Image\\Quality}
 & \makecell{Motion\\Smooth.}
 & \makecell{Aesthetic\\Quality}
 & \makecell{Temporal\\Flicker.}
 & \textbf{\makecell{Quality\\Score}} \\
\midrule
Self-Forcing    & \textbf{96.71} & \textbf{96.01} & 33.26 & \textbf{68.74} & \textbf{98.11} & \textbf{56.95} & 98.10 & 80.35 \\
Causal-Forcing  & 94.33 & 95.41 & 46.30 & 64.26 & 95.97 & 49.27 & 96.18 & 77.01 \\
\textbf{Ours}   & 94.57 & 94.89 & \textbf{85.23} & 65.44 & 97.66 & 54.69 & \textbf{98.10} & \textbf{82.64} \\
\bottomrule
\end{tabular}
\end{table}

\section{Additional Qualitative Results}
\label{app:visualization}

In this section, we provide additional qualitative comparisons between our proposed method and three representative distilled baselines: CausVid, Self Forcing, Causal Forcing. To comprehensively evaluate the generation capabilities, we select a diverse set of text prompts covering realistic human portraits, animals with complex accessories, and stylized 3D cartoon characters. 

As shown in the following visual comparisons (Figure \ref{fig:supp_helmet} to Figure \ref{fig:supp_squirrel}), existing methods often struggle with maintaining strict identity consistency and detail preservation across different frames or viewpoints. For instance, baseline methods tend to lose track of fine-grained attributes such as the shape of the camera harness (Figure \ref{fig:supp_corgi}) or the structure of the life vest (Figure \ref{fig:supp_otter}). In contrast, our approach demonstrates superior robust subject consistency, temporal coherence, and visual quality, successfully preserving the intricate details of the subjects regardless of stylistic variations or complex background dynamics.

\begin{figure*}[htpb]
    \centering
    \includegraphics[width=0.88\linewidth]{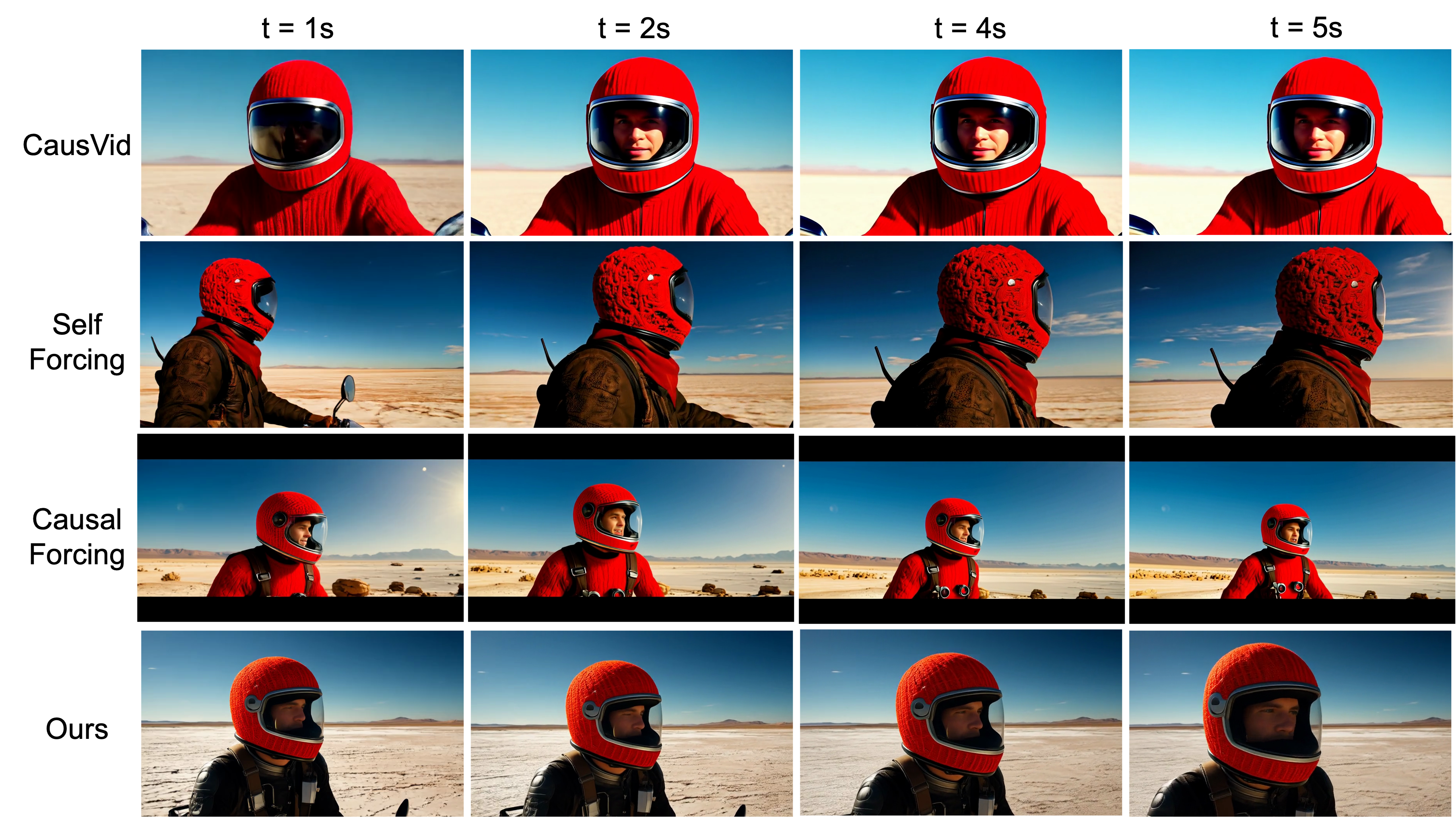}
    \caption{\textbf{Qualitative comparison on realistic human subjects.} The rows from top to bottom correspond to CausVid, Self Forcing, Causal Forcing and Ours. Our method maintains the structural integrity of the red helmet and the aesthetic quality significantly better than the baselines.}
    \label{fig:supp_helmet}
\end{figure*}

\begin{figure*}[htpb]
    \centering
    \includegraphics[width=0.88\linewidth]{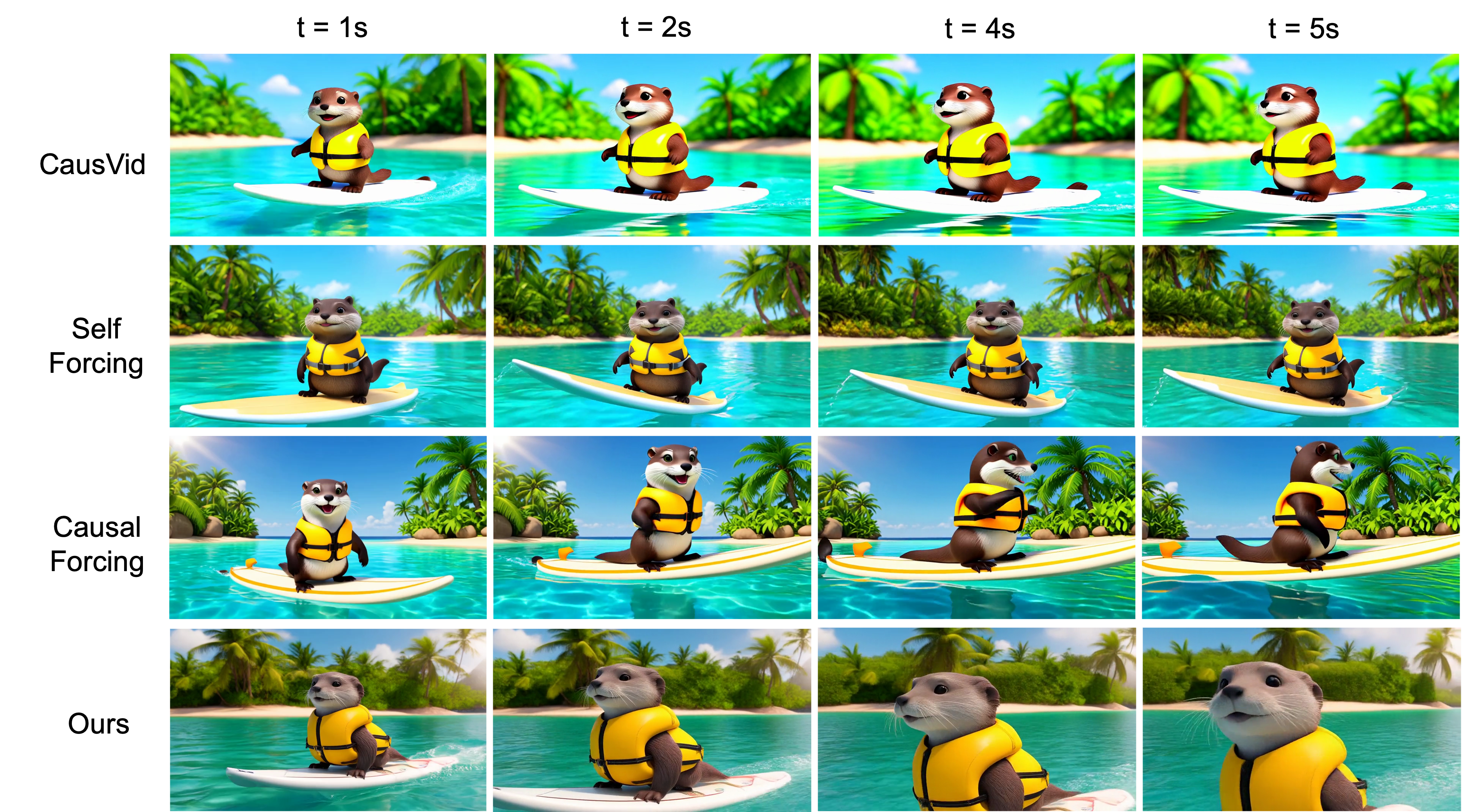}
    \caption{\textbf{Qualitative comparison on animals with accessories.} The baselines exhibit notable error accumulation, and identity drift. Our method ensures high-fidelity detail and temporal consistency.}
    \label{fig:supp_otter}
\end{figure*}

\begin{figure*}[htpb]
    \centering
    \includegraphics[width=0.88\linewidth]{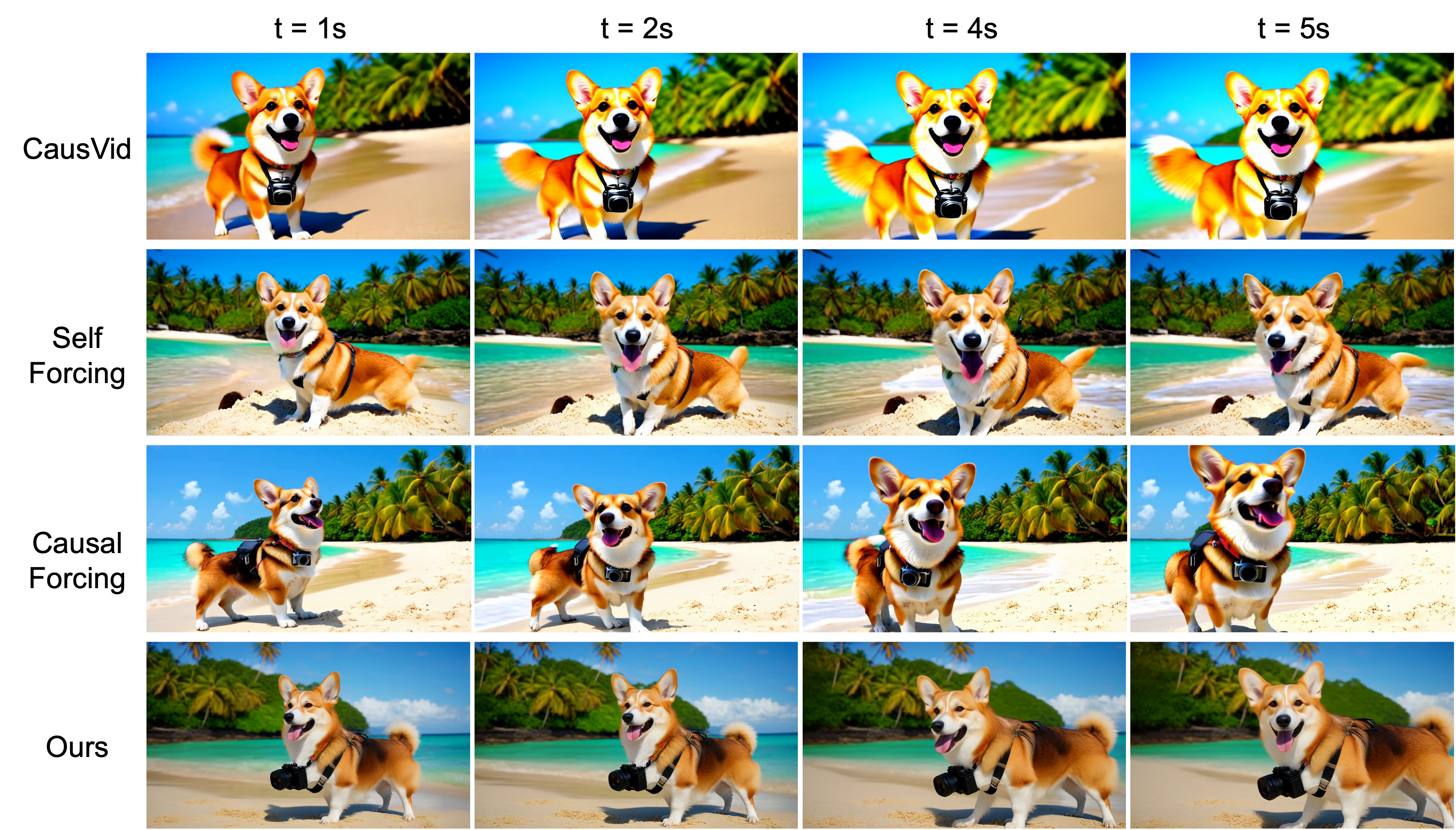}
    \caption{\textbf{Qualitative comparison on realistic animal portrait.} CausVid and Causal Forcing tend to suffer from the color oversaturation, While Self Forcing fails to generate the concept of \seedtextit{camera}. Our method shows better visual quality and prompt alignment.}
    \label{fig:supp_corgi}
\end{figure*}

\begin{figure*}[htpb]
    \centering
    \includegraphics[width=0.88\linewidth]{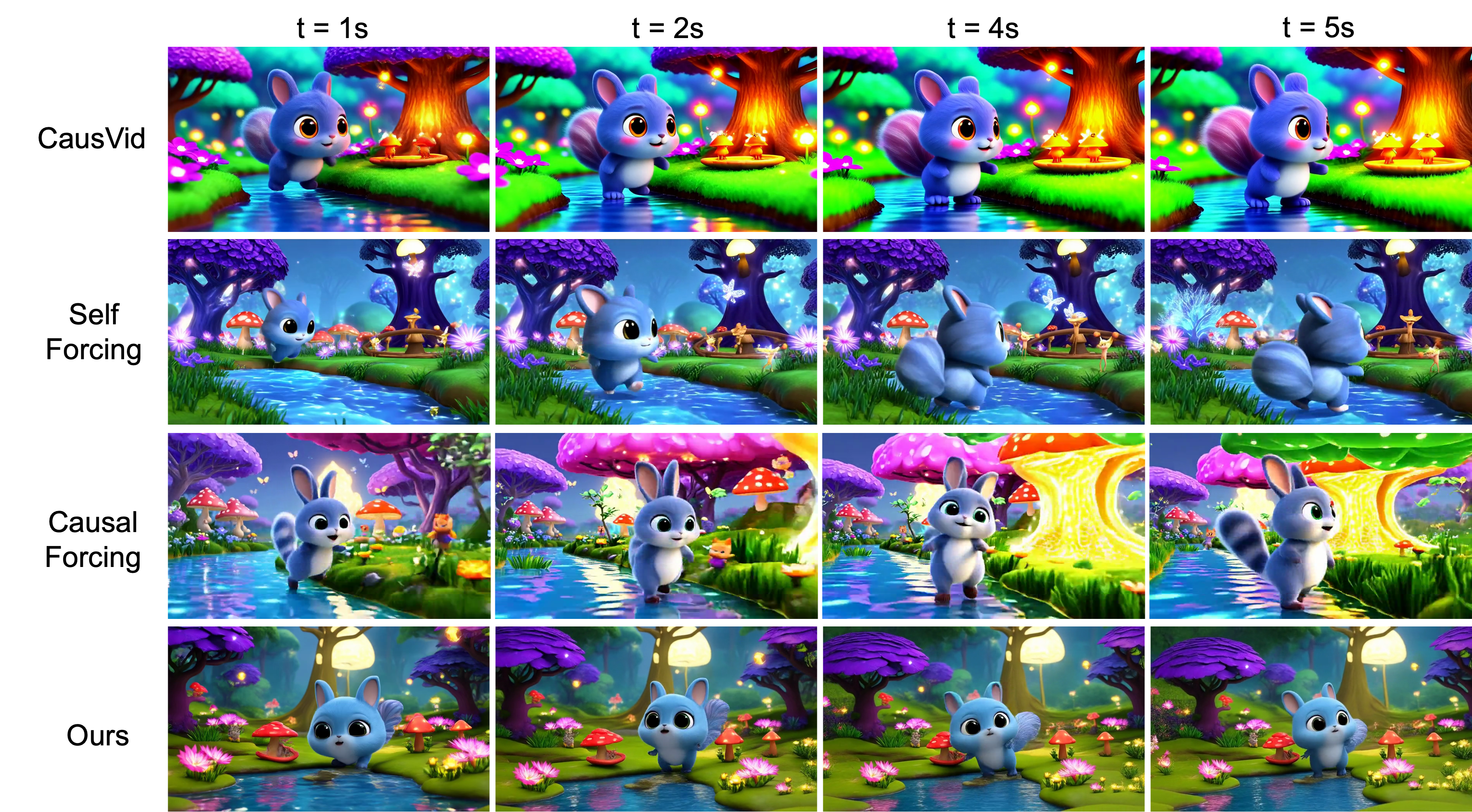}
    \caption{\textbf{Qualitative comparison on stylized 3D cartoon generation.} Our approach robustly generalizes to non-realistic/stylized domains. Baselines like CausVid and Causal Forcing show severe oversaturation and identity drift. By contrast, our method maintains consistent appearance and stable background details compared to the degraded outputs of existing methods.}
    \label{fig:supp_squirrel}
\end{figure*}


\end{document}